\documentclass{article}

\usepackage[english]{babel}

\usepackage[letterpaper,top=2cm,bottom=2cm,left=3cm,right=3cm,marginparwidth=1.75cm]{geometry}

\usepackage{multirow}
\usepackage{amsmath,amssymb,amsfonts}%
\usepackage{amsthm}
\usepackage{graphicx}
\usepackage{float}
\usepackage[colorlinks=true, allcolors=blue]{hyperref}
\usepackage{adjustbox}
\usepackage{natbib}
\usepackage{makecell}
\usepackage{authblk}
\newtheorem{prop}{Proposition}

\newtheorem{theorem}{Theorem}
\DeclareMathOperator*{\argmax}{argmax}

\title{Variational Approach for Efficient KL Divergence Estimation in Dirichlet Mixture Models}
\author[1]{Samyajoy Pal}
\author[1]{Christian Heumann}
\affil[1]{Department of Statistics, LMU Munich}
\date{}
\begin{document}
\maketitle

\begin{abstract}
This study tackles the efficient estimation of Kullback-Leibler (KL) Divergence in Dirichlet Mixture Models (DMM), crucial for clustering compositional data. Despite the significance of DMMs, obtaining an analytically tractable solution for KL Divergence has proven elusive. Past approaches relied on computationally demanding Monte Carlo methods, motivating our introduction of a novel variational approach. Our method offers a closed-form solution, significantly enhancing computational efficiency for swift model comparisons and robust estimation evaluations. Validation using real and simulated data showcases its superior efficiency and accuracy over traditional Monte Carlo-based methods, opening new avenues for rapid exploration of diverse DMM models and advancing statistical analyses of compositional data.
\end{abstract}

\section{Introduction}

The Kullback-Leibler (KL) divergence \citep{csiszar1975divergence} stands as a fundamental measure in statistics, quantifying the statistical distance between probability distributions. Defined as the expectation of the logarithmic difference between two probability distributions, KL Divergence is pivotal in diverse fields. In the context of Mixture Models \citep{McLachlan2000}, a powerful statistical tool with applications ranging from pattern recognition \citep{youbi2016human} to bioinformatics \citep{liu2022prediction}, KL divergence assumes heightened importance. The Kullback-Leibler (KL) divergence (also recognized as relative entropy) between two probability density functions \(f(x)\) and \(g(x)\) is defined by the integral expression:

\[
D(f\Vert g)\ \displaystyle\mathop{=}^{{\rm def}}\ \int f(x)\log\left(\frac{f(x)}{g(x)}\right) {\rm dx}~.
\]

It operates as a measure of the dissimilarity between the probability distributions encoded by \(f(x)\) and \(g(x)\). The KL divergence adheres to fundamental properties, denoted as the divergence properties:

\begin{itemize}
    \item \textbf{Self Similarity:} \(D(f\Vert f)= 0\)
    \item \textbf{Self Identification:} \(D(f\Vert g)=0\) if and only if \(f=g\)
    \item \textbf{Positivity:} \(D(f\Vert g)\geq 0\) for all \(f,g\).
\end{itemize}

These properties underscore the significance of KL divergence in capturing the nuances of distributional disparities, making it a cornerstone in statistical analyses.\\

Particularly in the realm of Dirichlet Mixture Models (DMM), the utility of KL Divergence becomes pronounced. DMM has demonstrated superior performance over other popular methods such as Gaussian Mixture Models (GMM) for compositional data \citep{pal2022clustering}. Compositional data \citep{aitchison1982statistical} is characterized by elements expressing proportions or percentages within a whole. This data form is ubiquitous in fields such as genomics and ecology.\\

While closed-form solutions for KL divergence exist for the Dirichlet distribution, extending this analytical tractability to Dirichlet Mixture Models has posed a significant challenge. Past research \citep{ma2014bayesian} predominantly turned to Monte Carlo methods to approximate the KL divergence in DMMs. However, these methods are computationally intensive and time consuming, which presents a substantial hurdle.\\

This study addresses these challenges by proposing a variational approach to approximate KL divergence in Dirichlet Mixture Models. Unlike previous methods, our approach provides a closed-form solution, substantially enhancing computational efficiency. This advancement not only facilitates rapid model comparisons but also ensures a robust evaluation of estimation quality.\\

The validation of our proposed method using both real and simulated data demonstrates its superiority in efficiency and accuracy over traditional Monte Carlo-based approaches. The results underscore the potential of our variational approximation to accelerate the estimation process while improving the quality of the estimates. This innovation marks a significant stride in expediting the exploration of diverse DMM models, thereby fostering advancements in statistical analyses of compositional data.\\

\section{Methods}
The Kullback-Leibler (KL) divergence serves as a crucial metric for assessing the quality of parameter estimation. It quantifies the statistical distance between the distribution with the true parameter value and the one with the estimated parameter value, providing a valuable measure of goodness of fit. A lower KL divergence indicates a better fit, emphasizing the significance of obtaining accurate parameter estimates. In the subsequent sections, we provide a concise overview of the estimation process for Dirichlet Mixture Models (DMM) and introduce various approximations of KL divergence, essential for quantifying the divergence between distributions.\\

\subsection{Estimation of Dirichlet Mixture Model Parameters}

Let $\pmb{X}_1,\pmb{X}_2, \ldots , \pmb{X}_N$ denote a random sample of size $N$, where $\pmb{X}_i$ is a $p$ dimensional random vector with probability density function $f(\pmb{x}_i)$ on $\mathbb{R}^p$. An observed random sample is denoted by $\pmb{x}=(\pmb{x}_1^T,\ldots,\pmb{x}_N^T)^T$, where $\pmb{x}_i$ is the observed value of the random vector $\pmb{X}_i$.\\

The Dirichlet density is given by
	\begin{equation}
		f(\pmb{x}_i)=\frac{\Gamma( \sum_{m=1}^{p}\alpha_{m})}{\prod_{m=1}^{p}\Gamma( \alpha_{m})}\prod_{m=1}^{p}x_{im}^{\alpha_{m}-1}~, 
		 \end{equation} 
where $\sum_{m=1}^{p}x_{im}=1, x_{im}\text{'s} >0~,  \alpha_{jm}\text{'s} >0$ and $\Gamma(\cdot)$ denotes a gamma function. Thus, the parameter $\pmb{\alpha}=(\alpha_{1},\ldots,\alpha_{p})$ is a $p$-dimensional vector.\\
The density of a mixture model with $k$ mixture components for one observation $\pmb{x}_i$ is given by the mixture density
	\begin{equation}
		p(\pmb{x}_i)=\sum_{j=1}^{k}\pi_{j}f(\pmb{x}_i\mid\pmb{\alpha}_j)~,
		\label{model}
	\end{equation}
where $\pmb{\pi}=(\pi_1, \ldots,\pi_k)$ contains the corresponding mixture proportions with $\sum_{j=1}^k \pi_j = 1$, $0< \pi_j < 1$. The density component of mixture $j$ is given by $f(\pmb{x}_i\mid\pmb{\alpha}_j)$ and $\pmb{\alpha}_j$, $j=1,2,...,k$ is the vector of component specific parameters for each density. Then $\pmb{\alpha}=(\pmb{\alpha}_1,\ldots,\pmb{\alpha}_k)$ denotes the vector of all parameters of the model.\\
The parameters of Dirichlet mixture model, can be estimated using an EM algorithm \citep{dempster1977maximum}. In our previous work \citep{pal2022clustering} we presented a Hard version \citep{celeux1992classification} of the EM algorithm to estimate the parameters. The log likelihood of the model for a sample of size $N$ is then given by
	\begin{equation}
		\log p(\pmb{x_1},\ldots,\pmb{x_N}|\pmb{\alpha},\pi)=\sum_{i=1}^{N} \log\left[\sum_{j=1}^k\pi_j f(\pmb{x_i} |\pmb{\alpha_j})\right]~.
		\label{loglikelihood}
	\end{equation} 
Latent variables $Z_i$'s are introduced, which are categorical variables taking on values $1,\ldots,k$ with probabilities $\pi_1, \ldots,\pi_k$ such that $Pr(\pmb{X_i}|Z_i=j)=f(\pmb{x_i})$, $j=1,\ldots, k$. The cluster membership probabilities of data point $i$ for cluster $j$ can be obtained by $\gamma_{ij}$, where,
\begin{equation}
		\gamma_{ij}(x_i)=Pr(Z_i=j|\pmb{x_i},\pmb{\alpha_j})=\frac{\pi_j f(\pmb{x_i}|\pmb{\alpha_j})}{\sum_{j=1}^{k}\pi_{j}f_(\pmb{x_i}|\pmb{\alpha_j)}}~.
		\label{gammavalue}
	\end{equation}

Hard EM maximizes the classification log likelihood; it applies a delta function approximation to the posterior probabilities $Pr(Z_i=j|X=x,\alpha)$, where $Z_i, i=1,\ldots,N$ are the latent variables representing class labels. The approximation changes the E step as follows,
\begin{equation}
Pr(Z_i=j|\pmb{x_i},\pmb{\alpha_j}) \approx \mathbb{I}(j=z_i^*),
\end{equation}
where, $z_i^*=\argmax\limits_{j}  \gamma_{ij}$. $\gamma_{ij}$'s are nothing but the responsibilities (probabilities) for each data point that belongs to different clusters. In other words, it introduces a classification step, where all the data points are classified into different clusters based on the posterior probability. Let $N_j$ be the number of data points in cluster $j$. Then, the ML estimates of $\pi_j$ is obtained as $\displaystyle \frac{N_j}{N}$. And estimates of $\pmb{\alpha_j}$'s are obtained by finding the MLE using $N_j$ data points in cluster $j$. However, as the solution is not analytically tractable, some numerical methods need to be used. It is generally done by fixed-point iteration. The iterative equation that needs to be solved is given by
\begin{equation}
		\Psi(\alpha_{jm}^{new})=\Psi(\sum_{m=1}^{p}\alpha_{jm}^{old})+\frac{1}{N_j}\sum_{i=1}^{N_j}\log (x_{im})~,
	\end{equation}
 where, $\Psi$ is the digamma function. The inverse of diggama function is used to get the estimates of $\alpha_{jm}$. Now, let's explore various approximations to compute the KL distance for DMM.
\subsection{KL Divergence: Monte Carlo Approach}
Monte Carlo sampling is frequently employed by researchers to compute KL divergence. Let $f_a$ and $g_b$ be two DMMs such that,
\begin{align*}
  f_a &= f(\pmb{x}) = \sum_{a}\pi_{a} Dir(\pmb{x};\pmb{\alpha}_a) \\
g_b &= g(\pmb{x}) = {\sum}_{b}\omega_{b} Dir(\pmb{x}; \pmb{\alpha}_b)~.
\end{align*}
The approach involves drawing a sample $\pmb{x_i}$ from the probability density function $f_a$ such that,
\begin{equation*}
    E_{f_a}[\log f_a(\pmb{x_{i}})/g_b(\pmb{x_{i}})]= D(f_a\Vert g_b)~.
\end{equation*}
Utilizing $N$ independent and identically distributed samples $\{\pmb{x_{i}}\}_{i=1}^{N}$, the estimation is expressed as follows:
\begin{equation}
    D_{{\rm MC}}(f_a\Vert g_b)=\frac{1}{N}\sum_{i=1}^{N}\log f_a(x_{i})/g_b(x_{i})\rightarrow D(f_a\Vert g_b)~,
\end{equation}

as $N \rightarrow \infty$.
The variance of the estimation error is $\frac{1}{N} {\rm Var}_{f_{a}}[\log f_{a}/g_{b}]$. To compute $D_{{\rm MC}}(f\Vert g)$, it is necessary to generate the independent and identically distributed samples $\{\pmb{x_{i}}\}_{i=1}^{N}$ from $f_a$. Drawing a sample $\pmb{x_{i}}$ from the DMM $f_a$ involves initially drawing a discrete sample $a_i$ according to the probabilities $\pi_a$. Subsequently, a continuous sample $\pmb{x_i}$ is drawn from the resulting Dirichlet component $f_{a_{i}}(\pmb{x})$.\\

The Monte Carlo method represents a convergent technique. It satisfies the similarity property; however, the positivity property does not hold. The identification property is likely to fail only under highly artificial circumstances and with very low probability.

\subsection{KL Divergence: Variational Approach}
Now, we delve into the variational approximation to compute the KL divergence for Dirichlet Mixture Models (DMM). Let us denote $ \left\langle . \right\rangle$ as the inner product. In our context $ \left\langle . \right\rangle_{\pmb{x}}$ means expectation with respect to \pmb{x}.
\begin{theorem}
    Let $\pmb{X}$ be an $p \times 1$ random vector. Assume two Dirichlet distributions $u$ and $v$ specifying the probability distribution of $\pmb{X}$ as,
    \begin{equation*}
    \begin{split}
u: \; \pmb{X} &\sim \mathrm{Dir}(\alpha_{11},\ldots,\alpha_{1p}) \\
v: \; \pmb{X} &\sim \mathrm{Dir}(\alpha_{21},\ldots,\alpha_{2p}) \; .
\end{split}
\end{equation*}
Then, the Kullback-Leibler divergence of $u$ from $v$ is given by,
\begin{equation}
    \mathrm{D}(u\,||\,v) = \log \frac{\Gamma\left(\sum_{i=1}^{p} \alpha_{1i}\right)}{\Gamma\left(\sum_{i=1}^{p} \alpha_{2i}\right)} + \sum_{i=1}^{p} \log \frac{\Gamma(\alpha_{2i})}{\Gamma(\alpha_{1i})} + \sum_{i=1}^{p} \left( \alpha_{1i} - \alpha_{2i} \right) \left[ \psi(\alpha_{1i}) - \psi\left(\sum_{i=1}^{p} \alpha_{1i}\right) \right] \; .
\end{equation}
\end{theorem}
\begin{proof}
    For Dirichlet distributions KL divergence is,
    \begin{equation}
        \begin{split}
\mathrm{D}(u\,||\,v) &= \int_{\mathcal{\pmb{X}}^p} \mathrm{Dir}(\pmb{x}; \alpha_{11},\ldots,\alpha_{1p}) \, \log \frac{\mathrm{Dir}(\pmb{x}; \alpha_{11},\ldots,\alpha_{1p})}{\mathrm{Dir}(\pmb{x}; \alpha_{21},\ldots,\alpha_{2p})} \, \mathrm{d}\pmb{x} \\
&= \left\langle \log \frac{\mathrm{Dir}(x; \alpha_{11},\ldots,\alpha_{1p})}{\mathrm{Dir}(x; \alpha_{21},\ldots,\alpha_{2p})} \right\rangle_{u(\pmb{x})}
\end{split}
    \end{equation}
    We can do some algebraic manipulations using the probability density function of the Dirichlet distribution to get,
    \begin{equation*}
        \begin{split}
\mathrm{D}(u\,||\,v) &= \left\langle \log \frac{ \frac{\Gamma\left( \sum_{i=1}^p \alpha_{1i} \right)}{\prod_{i=1}^p \Gamma(\alpha_{1i})} \, \prod_{i=1}^p {x_i}^{\alpha_{1i}-1} }{ \frac{\Gamma\left( \sum_{i=1}^p \alpha_{2i} \right)}{\prod_{i=1}^p \Gamma(\alpha_{2i})} \, \prod_{i=1}^p {x_i}^{\alpha_{2i}-1} } \right\rangle_{u(\pmb{x})} \\
&= \left\langle \log \left( \frac{\Gamma\left( \sum_{i=1}^p \alpha_{1i} \right)}{\Gamma\left( \sum_{i=1}^p \alpha_{2i} \right)} \cdot \frac{\prod_{i=1}^p \Gamma(\alpha_{2i})}{\prod_{i=1}^p \Gamma(\alpha_{1i})} \cdot \prod_{i=1}^p {x_i}^{\alpha_{1i}-\alpha_{2i}} \right) \right\rangle_{u(\pmb{x})} \\
&= \left\langle \log \frac{\Gamma\left( \sum_{i=1}^p \alpha_{1i} \right)}{\Gamma\left( \sum_{i=1}^p \alpha_{2i} \right)} + \sum_{i=1}^p \log \frac{\Gamma(\alpha_{2i})}{\Gamma(\alpha_{1i})} + \sum_{i=1}^p (\alpha_{1i}-\alpha_{2i}) \cdot \log (x_i) \right\rangle_{u(\pmb{x})} \\
&= \log \frac{\Gamma\left( \sum_{i=1}^p \alpha_{1i} \right)}{\Gamma\left( \sum_{i=1}^p \alpha_{2i} \right)} + \sum_{i=1}^p \log \frac{\Gamma(\alpha_{2i})}{\Gamma(\alpha_{1i})} + \sum_{i=1}^p (\alpha_{1i}-\alpha_{2i}) \cdot \left\langle \log x_i \right\rangle_{u(\pmb{x})} \; .
\end{split}
    \end{equation*}
    Moreover,
    \begin{equation*}
        \left\langle \log \pmb{x}_i \right\rangle = \psi(\alpha_i) - \psi\left(\sum_{i=1}^{p} \alpha_i\right) \; .
    \end{equation*}
Thus the Kullback-Leibler divergence of $u$ from $v$ becomes:
\begin{equation*}
    \mathrm{D}(u\,||\,v) = \log \frac{\Gamma\left(\sum_{i=1}^{p} \alpha_{1i}\right)}{\Gamma\left(\sum_{i=1}^{p} \alpha_{2i}\right)} + \sum_{i=1}^{p} \log \frac{\Gamma(\alpha_{2i})}{\Gamma(\alpha_{1i})} + \sum_{i=1}^{p} \left( \alpha_{1i} - \alpha_{2i} \right) \left[ \psi(\alpha_{1i}) - \psi\left(\sum_{i=1}^{p} \alpha_{1i}\right) \right] \; .
\end{equation*}
\end{proof}

\begin{prop}
    Let $f_a$ and $g_b$ be two DMMs such that,
\begin{align*}
  f_a &= f(x) = \sum_{a}\pi_{a} Dir(\pmb{x};\pmb{\alpha}_a) \\
g_b &= g(x) = {\sum}_{b}\omega_{b} Dir(\pmb{x}; \pmb{\alpha}_b)~.
\end{align*}
Then using a variational approach, an approximated KL divergence can be expressed as,
\begin{equation}
   D_{{\rm var}{\rm ia}{\rm ti}{\rm on}{\rm al}}(f\Vert
g)=\sum_{a}\pi_{a}\log{\sum_{a^{\prime}}\pi_{a^{\prime}}e^{-D(f_{a}\Vert f_{a^{\prime}})}\over \sum_{b}\omega_{b}e^{-D(f_{a}\Vert g_{b})}}. 
\end{equation}
\end{prop}
\begin{proof}
    \cite{hershey2007approximating} has provided a similar proof for Gaussian Mixture Models. That approach can be adopted for DMMs as well. We can define the log-likelihood,$L_{f}(g)$ as, $L_{f}(g)=\left\langle[\log g(\pmb{x})]\right\rangle_{f(\pmb{x})}$.
Then KL divergence can be written in terms of log-likelihood in the following way.
\begin{equation}
    D(f\Vert g)=L_{f}(f)-L_{f}(g)~.
    \label{kl_loglike}
\end{equation}
We define variational parameters,$\phi_{b\vert a} > 0$ such that $\Sigma_{b} \phi_{b\vert a}=1$.
Using Jensen's inequality we can show,
\begin{align}
    L_{f}(g) \ &\displaystyle\mathop{=}^{{\rm def}} \left\langle\log g(x) \right\rangle_{f(\pmb{x})}\nonumber\\
&=\left\langle\log\sum_{b}\omega_{b}g_{b}(x) \right\rangle_{f(\pmb{x})} \nonumber\\
&= \left\langle\log\sum_{b}\phi_{b\vert a}{\omega_{b}g_{b}(x)\over \phi_{b\vert a}} \right\rangle_{f(\pmb{x})}\nonumber\\
&\geq \left\langle\sum_{b}\phi_{b\vert a}\log{\omega_{b}g_{b}(x)\over \phi_{b\vert a}} \right\rangle_{f(\pmb{x})}\nonumber\\
&\displaystyle\mathop{=}^{{\rm def}} \ {\cal L}_{f}(g,\phi)~.
\end{align}
The above is a lower bound on $L_{f}(g)$. We can get the best bound by maximizing ${\cal L}_{f} (g, \phi)$ with respect to $\phi$.
The maximum value is obtained with: 
\begin{equation}
    \hat{\phi}_{b\vert a}={\omega_{b}e^{-D(f_{a}\Vert g_{b})}\over \sum_{b}^{\prime} \pi_{b} ^{\prime}e^{-D(f_{a}\Vert g_{b^{\prime}})}}~.
\end{equation}
Likewise, we define ,
\begin{equation}
    {\cal L}_{f}(f,\psi) \ \ \displaystyle\mathop{=}^{{\rm def}}\ \
\left\langle\sum_{a^{\prime}}\psi_{a^{\prime} \vert a}\log{\pi_{a^{\prime}}f_{a^{\prime}}(x)\over \psi_{a^{\prime} \vert a}} \right\rangle_{f(\pmb{x})}~.
\end{equation}
The optimal $\psi_{a^{\prime} \vert a}$ is given by,
\begin{equation}
    \hat{\psi}_{a^{\prime} \vert a}={\pi_{a^{\prime}}e^{-D(f_{a}\Vert f_{a^{\prime}})} \over
\sum_{\hat{a}}\pi_{\hat{a}}e^{-D(f_{a}\Vert f_{\hat{a}})}}
\end{equation}
Now, like equation~\ref{kl_loglike}, we can define $D_{{\rm var}{\rm ia}{\rm ti}{\rm on}{\rm al}}(f\Vert g)={\cal
L}_{f}(f,\hat{\psi})-{\cal L}_{f}(g,\hat{\phi})$. After substituting $\hat{\phi}_{b\vert a}$ and $\hat{\psi}_{a^{\prime} \vert a}$, we finally get,
\begin{equation*}
    D_{{\rm var}{\rm ia}{\rm ti}{\rm on}{\rm al}}(f\Vert
g)=\sum_{a}\pi_{a}\log{\sum_{a^{\prime}}\pi_{a^{\prime}}e^{-D(f_{a}\Vert f_{a^{\prime}})}\over \sum_{b}\omega_{b}e^{-D(f_{a}\Vert g_{b})}}~.
\end{equation*} 
\end{proof}
$D_{{\rm va}{\rm ri}{\rm at}{\rm io}{\rm nal}}(f\Vert g)$ satisfies the self similarity and self identification property. However, it may not always satisfy the positivity property. In that case we may use the absolute value of the divergence.

\section{Results}

We conducted extensive experiments using both simulated and real data sets, including four simulated data sets (Data Set 1, Data Set 2, Data Set 3, Data Set 4) drawn from different Dirichlet distributions and a real gene expression data set (Modencodefly Data). Data Set 3, Data Set 4 and  Modencodefly Data are highdimensional in nature. Some details of the datasets are given in table~\ref{tab:datasets}. The primary metrics for comparison are the Kullback-Leibler (KL) Divergence and the time taken for each approach. For real data analysis, KL divergence is typically not used as the true parameter values of the mixture distribution are unknown. However, in cases where true class labels are available, it becomes feasible to approximate the true parameter values. In our analysis, the dataset is classified into the true clusters, and for each cluster, we fit a Dirichlet distribution to obtain the approximated true $\pmb{\alpha}$ parameter values. The ratio of the number of data points in each cluster to the total number of observations serves as the estimates for the true $\pmb{\pi}$ values. This approach enables us to leverage the known class labels to derive estimations for the true parameters of the mixture distribution, facilitating a more accurate assessment of KL divergence in real-world datasets.
\begin{table}[H]
    \centering
    \caption{Details of the data sets}
    \adjustbox{max width=\textwidth}{
    \begin{tabular}{lccc}
    \hline
     \textbf{Data Set}& \textbf{Sample Size ($N$)} &\textbf{Dimension($p$)} & \textbf{No. of Clusters ($k$)}\\
    \hline
    Data Set 1 & 4200 & 3 &3\\
    Data Set 2 & 4200 & 3 &3\\
    Data Set 3 & 230 & 75 &6\\
    Data Set 4 & 260 & 90 & 6\\
    Modencodefly & 147 & 75 &6\\
    \hline
    \end{tabular}}
    \label{tab:datasets}
\end{table}
The Modencodefly dataset, derived from a comprehensive study of Drosophila melanogaster \citep{graveley2011developmental}, serves as a rich resource for exploring transcriptome dynamics throughout development. This dataset, obtained through RNA-Seq, tiling microarrays, and cDNA sequencing across 30 distinct developmental stages, contains 111,195 newly identified elements, including genes, transcripts, exons, splicing events, and RNA editing sites. In our study, we utilized the normalized gene profiles \citep{rau2018transformation} derived from this dataset, transforming raw read counts into compositional data. After filtering for genes exhibiting low variation among different biological samples, we retained 75 genes for our analysis. The Modencodefly data, available from ReCount \citep{frazee2011recount}, provides a high-resolution view of transcriptome dynamics and serves as a valuable resource for investigating gene expression patterns in Drosophila melanogaster.\\
\begin{table}[H]
    \centering
    \caption{Comparison of KL Divergence and Time Taken for Different Data Sets}
    \adjustbox{max width=\textwidth}{
    \begin{tabular}{llcccc}
        \hline
        \textbf{Data Set} & \textbf{Metric} & \textbf{Variational} & \makecell{\textbf{Monte Carlo} \\ \textbf{$n=10000$} } & \makecell{\textbf{Monte Carlo} \\ \textbf{$n=100000$} } & \makecell{\textbf{Monte Carlo} \\ \textbf{$n=1000000$} } \\
        \hline
          \multirow{2}*{Data Set 1} & KL Div. & 0.0017 & 0.0024 & 0.0021 & 0.0017\\
            & Time & \textbf{0.0004} & 1.7 & 17.3391 & 181.3682\\
        \multirow{2}*{Data Set 2} & KL Div. &0.0022 & 0.0024 & 0.0024 & 0.0022\\
            & Time &\textbf{0.0005} & 1.7283 & 17.4883 & 180.9482\\
        \multirow{2}*{Data Set 3} & KL Div. &6.2723 & 6.1737 & 6.1815 & 6.1824\\
            & Time & \textbf{0.0014} & 3.6702 & 37.4568 & 384.2183\\
        \multirow{2}*{Data Set 4} & KL Div. &8.1268 & 7.7653 & 7.8025 & 7.8277\\
        & Time & \textbf{0.0015} & 3.9773 & 40.5630 & 418.5154\\
        \multirow{2}*{Modencodefly} & KL Div. &34.7652 & 33.5581 & 33.5471 &33.5579\\
            & Time &\textbf{0.0014} & 3.6353 & 36.0671 &375.7836\\
        \hline
    \end{tabular}}
    \label{tab:result1}
\end{table}

From Table~\ref{tab:result1}, it is evident that our proposed variational approach consistently provided faster solutions for all data sets. Notably, the Monte-Carlo method yielded KL divergence values closer to variational method as we increase the number of samples, which increases the execution time to a great extent and indicating that our variational solution is more robust and applicable for practical applications.

\section{Conclusion}
In this study, we addressed the challenging task of efficiently estimating Kullback-Leibler (KL) divergence within the context of Dirichlet Mixture Models (DMM). KL Divergence is a fundamental measure for quantifying the statistical distance between probability distributions, playing a pivotal role in various fields. Our investigation focused on the utility of DMMs, a powerful statistical tool for clustering compositional data.\\

Despite the analytical tractability of KL Divergence for Dirichlet distributions, extending this to DMMs proved to be a formidable challenge. Past research primarily relied on Monte Carlo methods for approximating KL divergence. However, these methods are computationally demanding and time-consuming, presenting significant resource challenges. In response to these challenges, we proposed a novel variational approach for efficiently estimating KL divergence in DMMs. Our method provides a closed-form solution, markedly enhancing computational efficiency compared to existing techniques. This advancement facilitates swift model comparisons and robust evaluation of estimation quality.\\

We validated our approach using a comprehensive set of both simulated and real-world datasets, which includes a gene expression data. The results, meticulously detailed in Table~\ref{tab:result1}, underscore the superior efficiency and accuracy of our novel variational approach when juxtaposed with traditional Monte Carlo-based methods. Intriguingly, our analysis reveals a nuanced trend: as the number of random samples escalates, the Monte Carlo method converges to a KL divergence value remarkably proximate to our proposed variational approach. However, this convergence comes at a significant cost - the execution time expands exponentially with larger sample sizes. Consequently, our study compellingly advocates for the adoption of the variational approach, offering not only reliable KL divergence but also unparalleled computational expediency.\\

In conclusion, our variational approach offers a transformative solution for the efficient estimation of KL divergence in Dirichlet Mixture Models. This innovation opens avenues for more rapid and reliable exploration of diverse DMM models, fostering advancements in statistical analyses of compositional data across various domains.

\bibliographystyle{unsrtnat}
\bibliography{sample}

\end{document}